%% file: arxiv_version.tex
\documentclass{article}

\include{header}

\include{macro}

\usepackage{arxiv}

\usepackage[utf8]{inputenc} 
\usepackage[T1]{fontenc}    
\usepackage{hyperref}       
\usepackage{url}            
\usepackage{booktabs}       
\usepackage{nicefrac}       
\usepackage{microtype}      
\usepackage{natbib}
 \PassOptionsToPackage{compress, round}{natbib}
\usepackage{comment}

\title{Byzantine-Robust Federated Linear Bandits}

\begin{document}

\author{\name Ali Jadbabaie\footnotemark[1] \email{jadbabai@mit.edu} \\
\name Haochuan Li\footnotemark[1] \email{haochuan@mit.edu} \\
\name Jian Qian\footnotemark[1] \email{jianqian@mit.edu} \\
\name Yi Tian\footnotemark[1] \email{yitian@mit.edu} \\
\addr{Massachusetts Institute of Technology}
}

\renewcommand*{\thefootnote}{\fnsymbol{footnote}}
\footnotetext[1]{Supported in parts by ONR grants N00014-20-1-2336 and N00014-20-1-2394, an MIT-IBM Watson grant, and NSF BIGDATA grant 1741341.}
\renewcommand*{\thefootnote}{\arabic{footnote}}

\maketitle

\input{abstract}

\input{1_intro}
\input{2_pre}
\input{3_alg}

\input{4_analysis}

\input{5_dp}

\input{6_mom}

\input{7_conclusion}

\bibliographystyle{plainnat}
\bibliography{reference}
\clearpage
\appendix
\input{app_lemmas}
\input{app_proofs}
\input{app_proof_MoM}

\end{document}

%% file: header.tex
\usepackage{amsmath} 
\usepackage{amssymb} 
\usepackage{amsthm}
\usepackage{color}
\usepackage[usenames,dvipsnames]{xcolor}
\usepackage{multicol}
\usepackage{enumitem}
\usepackage{bbm}
\usepackage{mathtools}
\usepackage{fancyhdr}
\usepackage{hyperref}
\usepackage{listings}
\usepackage{mathrsfs}
\usepackage{bm}
\usepackage{xspace}
\usepackage{algorithm, algorithmicx, algpseudocode}

%% file: macro.tex
\newcommand{\haochuan}[1]{}

\newcommand{\norm}[1]{\left\|#1\right\|}
\newcommand{\inner}[1]{\langle #1 \rangle}
\newcommand{\abs}[1]{\left| #1 \right|}

\newcommand{\epiLen}{L}
\newcommand{\numEpi}{K}

\newcommand{\mnoise}{H}
\newcommand{\mbd}{B}
\newcommand{\vnoise}{h}
\newcommand{\vbd}{B}
\newcommand{\thetabd}{\sqrt{d}}
\newcommand{\lse}{\text{lse}}
\newcommand{\sigmax}{\sigma}

\DeclareMathOperator*{\argmax}{argmax}
\DeclareMathOperator*{\argmin}{argmin}
\DeclareMathOperator*{\tr}{Trace}
\DeclareMathOperator{\GM}{GM}

\newtheorem{lemma}{Lemma}
\newtheorem{assumption}{Assumption}
\newtheorem{proposition}[lemma]{Proposition}
\newtheorem{theorem}[lemma]{Theorem}
\newtheorem{remark}{Remark}

\newtheorem{definition}{Definition}

\newcommand{\mR}{\hat{R}}
\newcommand{\eR}{\tilde{R}}

\newcommand{\R}{\mathbb{R}}

\newcommand{\N}{\mathbb{N}}

\newcommand{\E}{\mathbb{E}}

\newcommand{\bR}{\mathbb{R}}

\usepackage{comment}

\usepackage{mathtools}
\DeclarePairedDelimiter{\ceil}{\lceil}{\rceil}
\DeclarePairedDelimiter{\floor}{\lfloor}{\rfloor}
\newcommand{\cA}{\mathcal{A}}

\newcommand{\cD}{\mathcal{D}}

\newcommand{\cG}{\mathcal{G}}

\newcommand{\cN}{\mathcal{N}}
\newcommand{\cO}{\mathcal{O}}
\newcommand{\cP}{\mathcal{P}}

\newcommand{\cT}{\mathcal{T}}

\newcommand{\cZ}{\mathcal{Z}}

\newcommand{\var}{\operatorname{Var}}

\newcommand*{\sheafhom}{\mathcal{H}\kern -.5pt \it{ om}}
\newcommand*{\sheafspec}{\mathcal{S}\kern -.5pt \it{ pec}}
\newcommand*{\sheafproj}{{\mathcal{P}\kern -.5pt \it{ roj}}}

\newcommand{\epsilonGM}{\epsilon}

\newcommand{\linucb}{\texttt{LinUCB}\xspace}

\newcommand{\byucb}{\texttt{Byzantine-UCB}\xspace}
\newcommand{\byucbdp}{\texttt{Byzantine-UCB-DP}\xspace}
\newcommand{\byucbdpm}{\texttt{Byzantine-UCB-DP-MoM}\xspace}

\newcommand{\Agg}{\texttt{Aggregate}\xspace}
\newcommand{\Pri}{\texttt{Privatize}\xspace}

\newcommand{\dpepsilon}{\mu}
\newcommand{\dpdelta}{\nu}

\newcommand{\epiGram}{U}
\newcommand{\epiFeatSum}{u}
\newcommand{\Gram}{V}
\newcommand{\featSum}{v}

\newcommand{\groupnumber}{P}
\newcommand{\group}{\cP}

%% file: abstract.tex
\begin{abstract}
  In this paper, we study a linear bandit optimization problem in a federated setting where a large collection of distributed agents collaboratively learn a common linear bandit model. Standard federated learning algorithms applied to this setting are vulnerable to Byzantine attacks on even a small fraction of agents. We propose a novel algorithm with a robust aggregation oracle that utilizes the geometric median. We prove that our proposed algorithm is robust to Byzantine attacks on fewer than half of agents and achieves a sublinear $\tilde{\mathcal{O}}({T^{3/4}})$ regret with $\mathcal{O}(\sqrt{T})$ steps of communication in $T$ steps. Moreover, we make our algorithm differentially private via a tree-based mechanism. Finally, if the level of corruption is known to be small, we show that using the geometric median of mean oracle for robust aggregation further improves the regret bound.
\end{abstract}

%% file: 1_intro.tex
\section{Introduction}

Recommendation systems have been a workhorse of e-commerce~\citep{Schafer1999RecommenderSI} and operations management applications~\citep{10.1007/978-3-319-07863-2_55} for more than a decade. The explosion of interest in personalized recommendation systems, however, has raised critical ethics and privacy issues. These trends, together with recent advances in federated and distributed computation, have given rise to new challenges and opportunities for the design of new recommendation systems that are developed using a secure, private, and federated architecture.

A key ingredient of such a system, at its core would be a bandit optimization engine. To this end, the current paper is motivated by the consideration of data corruption in a federated recommendation system. The recommendation system is modeled by a linear bandit with time-varying decision sets~\citep{abbasi2011improved}. The data corruption is modeled by the Byzantine attack~\citep{lamport1982byzantine}, a famous error model in distributed systems where parts of the system fail and there is imperfect information about the occurrences of the failures. How does one design provably robust algorithms in such a scenario?

More specifically, consider the scenario where one has to make recommendations to many devices. It is natural to assume that a device is continually used by the same user and that the users at different devices share similarity (e.g., from the same user group) so that at time step $t$, the decision sets $\cD_{i}^t$ for the devices are drawn i.i.d. from a distribution where $i$ denotes a device index. The distribution is unknown and can change over time, modeling the fact that the preferences of the user group may be influenced by certain events as time goes by. We take the linear bandit model~\citep{abbasi2011improved}, according to which after making recommendation $x_{i}^t \in \cD_{i}^t$, the reward we receive satisfies $\E[r_{i}^t \vert x_{i}^t] = x_{i}^t \cdot \theta^{\ast}$. Such a model is a special case of federated linear bandits~\citep{dubey2020differentially}. 

We consider a centralized federated learning setup~\citep{kairouz2019advances}, where devices are distributed and communicate with a central controller. Either due to noncooperative user behaviors or due to hijacking of the device by some adversary, the communications from some devices to the controller may be corrupted. Hence, it is vital that the federated recommendation system is robust to such corruptions.  
Here we consider a rather general and classical data corruption scheme called the Byzantine attack~\citep{lamport1982byzantine}, where the corrupted information is arbitrary and we have no knowledge about whether the corruption happens at a particular device.
Such a scenario has been considered in federated optimization~\citep{pillutla2019robust,wu2020federated}, where the performance of an algorithm is measured by the convergence rate. However, it is unclear how federated recommendation systems can be made robust to such attacks, where the performance of an algorithm is measured by the notion of regret. 

An immediate question is how we should define regret in such a scenario. Since corrupted devices may fail arbitrarily under the Byzantine attack, a reasonable way is to consider the regret defined on the uncorrupted devices, which we call \emph{robust regret}. A robust algorithm is then one that achieves sublinear robust regret. Since the controller has no information about which device has failed, it is challenging to design algorithms robust to the Byzantine attack. 

In this paper, we design an algorithm that is robust to such attacks under the above federated linear bandit model. Notably, we show that so long as more than half of the devices are consistently reliable, our algorithm, called \byucb, achieves $\tilde{\cO}(dNT^{3/4})$ robust regret for $N$ federated linear bandits of dimension $d$ in $T$ steps with $\cO(\sqrt{T})$ steps of communication. 
Essential to achieve robustness is the i.i.d.~assumption on the decision sets for different devices, since under such an assumption, reliable information can be obtained via robust estimation; specifically, by using the (geometric) median estimator in place of the mean estimator~\citep{minsker2015geometric}.

Although it is well-known that geometric median provably robustify the convergence of federated optimization, things are very different for the federated bandit problem. Unlike optimization where the geometric median is used to robustly estimate the mean of gradients, the linear bandit problem does not involve gradients. Instead, the challenge of the bandit problem is the well-known exploration-exploitation dilemma, where the agent attempts to acquire new knowledge (called "exploration") or to optimize its decisions based on existing knowledge (called "exploitation"). It becomes more challenging in a federated setting with Byzantine attacks. To our knowledge, we are the first to tackle this challenge.\haochuan{Added  hardness here}

Privacy preservation is a major concern in federated learning~\citep{Yin2021ACS} and one of the key contributions of this paper: since the users do not want other users to learn any of their personal information from the broadcast messages. \citet{dubey2020differentially} consider differentially private federated linear bandits, where differential privacy is defined for decision sets and rewards. Here with the Byzantine attack, messages that contain information in several steps can be manipulated. Therefore, it is more meaningful to consider a more general notion of differential privacy that is defined for the communication messages. Equipped with the tree-based mechanism~\citep{dwork2010differential,chan2011private}, our new algorithm (called \byucbdp) simultaneously achieves differential privacy for communication and a slightly worse $\tilde{\cO}(d^{3/2} N T^{3/4})$ robust regret.

Both \byucb and \byucbdp have the advantage of being agnostic to the proportion of devices that are corrupted. If the corruption proportion $\alpha$ is small and its upper bound is known, another robust estimation can be obtained by the median of mean approach~\citep{Darzentas1984}. \byucbdpm, with a different aggregation oracle based on median of mean, interpolates the robust regret between $\tilde{\cO}(d^{3/2} N T^{1/2})$ and $\tilde{\cO}(d^{3/2} N T^{3/4})$ for $0\le \alpha \le 1/2$ under the differential privacy constraint.

\paragraph{Summary of our contributions.}
In this section, we summarize the key contributions of the paper.  The first contribution is in problem statement and modeling: we introduce the problem of federated linear bandits under the Byzantine attack. To justify the necessity of federated learning, we show that the robust regret can be linear in the number of time steps $T$ for any algorithm without communication (Proposition~\ref{prop:lower_bound}). Furthermore, we present a federated learning algorithm called \byucb and two variants (\byucbdp and \byucbdpm) that have the following properties: 
\begin{itemize}[leftmargin=1.5em]    \setlength{\itemsep}{-1pt}
    \item \byucb achieves a sublinear $\tilde{\cO}(T^{3/4})$ robust regret in $T$ steps with $\cO(\sqrt{T})$ steps of communication (Theorem~\ref{thm:regret}).
    
    \item \byucbdp simultaneously guarantees differential privacy with a slightly worse robust regret and the same communication cost (Theorem~\ref{thm:regret_dp}).
    
    \item If knowledge about the proportion of corrupted devices is available, \byucbdpm, apart from the differential privacy guarantee, interpolates the robust regret between $\tilde{\cO}(\sqrt{T})$ and $\tilde{\cO}(T^{3/4})$ depending on the corruption proportion (Theorem~\ref{thm:regret_mom}).
    
\end{itemize}

\subsection{Related work}

\paragraph{Linear bandits.} 
\citet{auer2002using} introduced the first finite-time regret analysis of linear bandit under the name "linear reinforcement learning". The setting is then extensively studied \citep{abe2003reinforcement, dani2008stochastic,  abbasi2011improved}. Notably, \linucb by \citet{abbasi2011improved} forms the basis of our analysis. Moreover, this setting found its application in recommender systems \citep{li2010contextual, chu2011contextual}.

\paragraph{Federated learning.} Federated learning is a machine learning technique that trains an algorithm across multiple decentralized edge devices or servers holding local data samples, without exchanging them \citep{kairouz2019advances}. Progress has been made in the federated learning setting in distributed supervised learning \citep{konevcny2016federatedB} and federated optimization \citep{konevcny2016federated,Jadbabaie2022FederatedOO,Reisizadeh2020FedPAQAC,Reisizadeh2020RobustFL}. 
Many recent works study different aspects of the bandit problem or the more general reinforcement learning problem in a federated setting~\citep{dubey2020differentially,Dubey2020PrivateAB,Li2022CommunicationEF,Huang2021FederatedLC,Shi2021FederatedMB,Shi2021FederatedMBPers,Zhu2021FederatedB,Tao2021OptimalRO,Fan2021FaultTolerantFR}.

\paragraph{Differential privacy.} Privacy issues are also important for distributed systems. \citet{dwork2008differential,dwork2014algorithmic} introduced a cryptographically-secure privacy framework that characterized the privacy issue of an algorithm as the change in the output with a slight change in the input. Moreover, the tree-based algorithm is proposed by \citet{dwork2010differential, chan2011private} to realize the privacy requirement for partial sums. It is then applied to contextual bandits by \citet{shariff2018differentially,dubey2020differentially}.

\paragraph{Byzantine-robustness.} Byzantine attack is a type of attack that causes parts of a distributed system to fail while unknown to the other parts~\citep{lamport1982byzantine}. More specifically, an attacked part may behave completely arbitrarily and can send any message to other parts. Federated learning algorithms, working in a distributed manner, might suffer from such issues too. The key to resolving these issues is robust estimation which is pioneered by \citet{huber1992robust, huber2004robust}. \citet{Darzentas1984} first introduced the median of mean approaches. Recent works~\citep{Hsu2016a,Lecue2020,Lugosi2019,Lugosi2020, minsker2015geometric, pillutla2019robust} abound in the field of robust estimation.  Moreover, many such approaches are applied to the distributed optimization tasks against Byzantine attacks, where the target is mainly to improve stochastic gradient descent solver of the underlying optimization task, e.g. through aggregating by geometric median~\citep{chen2017distributed, wu2020federated}, median \citep{xie2018generalized}, trimmed median~\citep{yin2018byzantine}, iterative filtering~\citep{su2018securing}, Krum~\citep{Blanchard2017} and RSA~\citep{Li2019} etc. Of all these methods, we focus on geometric median and geometric median of mean for aggregation. Also, $\epsilon$-approximation of the two values are considered for tractability~\citep{pillutla2019robust}.

Two recent papers~\citep{Dubey2020PrivateAB,Fan2021FaultTolerantFR} also study fault-tolerant federated bandit or reinforcement learning problems. However, we focus on very different settings and aspects. First, \citet{Dubey2020PrivateAB} studies the multi-armed bandit problem which is more specific and easier than the linear bandit problem we consider. Moreover, they use a different corruption model where corrupted data are assumed to follow a fixed but unknown distribution while we consider arbitrary attacks. \citet{Fan2021FaultTolerantFR} studies the more general federated reinforcement learning problem. However, they are considering convergence to stationary points which is much weaker than the regret bounds considered in this paper.
\haochuan{added some related works}

%% file: 2_pre.tex
\section{Preliminaries}

\label{sec:problem_setup}

\paragraph{Notation.} For any integer $n\in\N$, let $[n]$ be the set $\{1,\ldots,n\}$. For a vector $x$, we use $x_i$ to denote its $i$-th coordinate and $\norm{x}_2$ to denote its $\ell_2$ norm. Given a positive semi-definite matrix $A$, we denote $\norm{x}_{A}=\sqrt{x^\top A x}$. For a matrix $A$, we denote its spectral norm and Frobenius norm by $\norm{A}_2$ and $\norm{A}_F$ respectively. Given two symmetric matrices $A$ and $B$ with the same size, we write $A< B$ or $B>A$ if $B-A$ is positive definite. We also write $A\le B$ or $B\ge A$ if $B-A$ is positive semi-definite.  Finally, we use the standard $\cO(\cdot)$, $\Theta(\cdot)$ and $\Omega(\cdot)$ notation, with $\tilde{\cO}(\cdot)$, $\tilde{\Theta}(\cdot)$, and $\tilde{\Omega}(\cdot)$ further hiding logarithmic factors. 

\subsection{Problem setup}
\label{subsec:problem_setup}

\paragraph{Federated learning under Byzantine attacks.}
We consider the federated environment where there is one central server and $N$ distributed agents. We assume that communication happens only between the controller and each agent. Let $\cN$ be the set of all agents with $\abs{\cN}=N$. At each time $t\in [T]$, several agents may be subject to a  Byzantine attack and try to send arbitrarily corrupted information to the controller. Let $\cN_0^t$ and $\cN_1^t$ denote the set of noncorrupted and corrupted agents at time $t$ respectively. Here we say an agent is reliable if it does not get attacked and corrupted otherwise. Let $\cN_0 = \bigcap_t \cN_0^t$ be the set of consistently reliable agents and $N_0 = |\cN_0|$. Also define $\cN_1 = \bigcup_t \cN_1^t$ be the complement of $\cN_0$ and $N_1 = |\cN_1|$. Assume at least half of the agents are consistently reliable, i.e., $\alpha \triangleq N_1 / N < 1/2$. 

\paragraph{Federated linear bandits.}
At every time $t\in [T]$, each noncorrupted agent $i\in\cN_0^t$ is presented with a decision set $\cD_i^t\subseteq\mathbb{R}^d$. It selects an action $x_i^t$ from $\cD_i^t$ and receives a reward $r_i^t=\langle x_i^t, \theta^\ast\rangle + \eta_i^t$, where $\theta^\ast\in\R^d$  is some unknown parameter and $\eta_i^t$ is a noise. We assume the decision set and the true parameter are bounded, i.e., $\max_{x\in \cD_i^t}\norm{x}_2\le  1$, $\|\theta^\ast\|_2\le \thetabd $.
To see why the assumption that $\|\theta^\ast\|_2\le \sqrt{d}$ is reasonable, consider multi-armed bandits, a special case of linear bandits, where standard bounded average reward assumption implies that $\|\theta^\ast\|_2 = \cO(\sqrt{d})$. 

The randomness of the model comes from $\cD_i^t$ and $\eta_i^t$ on which we make the following assumptions:
at each time step $t$, the pairs $\{(\cD_i^t,\eta_i^t)\}_{i\in\cN_0^t}$ are i.i.d. sampled from the an unknown distribution $\cP_t$ conditioned on previous $\{(\cD_i^s,\eta_i^s)\}_{i\in\cN_0^s}$ for all $s<t$. Also, assume that $\cD_i^t$ and $\eta_i^t$ are independent for each $t\in[T]$ and $i\in\cN_0^t$. Let $\cP_\eta^t$ be the marginal distribution of $\eta_i^t$. We assume $\cP_\eta^t$ is $R$-subGaussian. Note that the independence between $\cD_i^t$ and $\eta_i^t$ is assumed for ease of exposition and can be relaxed as in \citep{abbasi2011improved}.

Since we should not expect to pull the right arm on the corrupted steps, the objective of the agents is thus to minimize the cumulative pseudo-regret on the steps where they are not attacked. Formally, we define the regret as follows:

\begin{align}
    {R}_T 
    = \sum\nolimits_{t=1}^T \sum\nolimits_{i\in \cN_0^t} \left( \max_{x\in \cD_i^t}\inner{x, \theta^*}   - \inner{x_i^t,\theta^*} \right).
\end{align}

Note that in the presence of corruptions, if each agent learns its own problem without collaborating with others, the regret will be linear in $T$, as shown in Proposition~\ref{prop:lower_bound}. 
\begin{proposition}
    \label{prop:lower_bound}
    For a given set of agents $\cN$, corruption level $\alpha > 0$, there exists an instance of a federated linear bandit problem with corruptions under our assumptions such that without communication, ${R}_{T} \geq c\alpha NT$ for some absolute constant $c > 0$.
\end{proposition}

See Appendix~\ref{app:pf-prop} for the proof of Proposition~\ref{prop:lower_bound}.
Therefore, it is necessary to learn in a federated way. In the next section, we propose a federated algorithm which achieves a regret of $\tilde{\cO}(T^{3/4})$.

\subsection{Robust aggregation}
\label{subsec:robust_agg}
Standard federated learning algorithms cannot be applied in the current setting because they are vulnerable to Byzantine attacks. For example, \citet{dubey2020differentially} proposed an algorithm for federated linear bandits without corruptions. In their algorithm, the updates collected from agents are aggregated by a simple arithmetic mean, which is known to be vulnerable to Byzantine attacks on even a single agent. 

To robustify the algorithm, we utilize the  geometric median in our aggregation rule. For a collection of vectors $z_1, \ldots, z_n\in\R^d$, let $g(z)=\frac{1}{n}\sum_{i\in [n]} \|z - z_i\|_2$. We define $\GM_{i\in[n]}(z_i)\triangleq \argmin_{z\in\R^d} g(z)$ as their geometric median. In practice, this minimization problem is usually solved approximately. Hence, we further define the $\epsilonGM$-approximate geometric median as an approximate solution $\hat{z}$ satisfying $g(\hat{z})\le \min_{z\in\R^d} g(z)+\epsilonGM$ which we will denote by $\GM^{\epsilonGM}_{i\in[n]}(z_i)$. Note that if some attacked $z_i$ is even not a vector in $\R^d$, we view it as $0\in \R^d$ when computing the geometric median. We can also define the geometric median of matrices by replacing the $\ell_2$ vector norm with the Frobenius norm.  Recently, \citet{pillutla2019robust} proposed a robust oracle based on a smoothed Weiszfeld algorithm which returns an approximate geometric median with only a small number of calls to the average oracle. We will also adopt this robust oracle in our algorithm.

%% file: 3_alg.tex
\section{The \byucb algorithm and its robust regret bound}
\label{sec:alg}

In this section, we present our algorithm \byucb with the achieved regret bound. 
In general, our algorithm obtains a regret bound using a reasonable amount of communication between the agents and the controller while being robust to Byzantine attacks. 
Before introducing \byucb, we first introduce our general algorithmic framework (Algorithm~\ref{alg:byzantine-ucb}) for Byzantine-robust federated linear bandit optimization.

\paragraph{Algorithmic framework.}
In Algorithm~\ref{alg:byzantine-ucb}, to reduce the amount of communication, we divide the $T$ steps into $\numEpi$ episodes of length $\epiLen$, where $\numEpi, \epiLen \in\mathbb{N}$. Assume $T = \numEpi \epiLen$ exactly holds for simplicity; in general we can round up $K$ to the closest integer. At the start of each episode, the controller synchronizes the parameters $\theta_k$ and $\Lambda_k$ with all agents. Define $\cT_{\text{sync}} = \{1, \epiLen+1, \ldots, (\numEpi -1)\epiLen + 1\}$ as the set of steps when communication happens. For each $k\in[\numEpi]$, define $\cT_k = \{(k-1)\epiLen + 1, \ldots, k\epiLen\}$ as the set of steps between the $k$-th and $(k+1)$-th communication. 

During the $k$-th episode, agent $i$ runs in the same fashion as the celebrated  \linucb algorithm~\citep{abbasi2011improved}. The idea is to construct a confidence region which contains $\theta^\ast$ with high probability, and then to follow the principle of optimism in the face of uncertainty. Specifically, the confidence region is constructed as $\Theta_k = \{\theta \in \bR^d : \norm{\theta - \theta_k}_{\Lambda_k} \leq \beta_k  \} $, where $\beta_k$ is stored locally and specified by the algorithmic instantiation. Then, on receiving the decision set $\cD_i^t$ from the environment, the agent picks the most optimistic choice $x_i^t = \argmax_{x\in \cD_i^t} \sup_{\theta\in \Theta_k} \inner{x, \theta} = \argmax_{x\in \cD_i^t} \inner{x,\theta_k} + \beta_k \norm{x}_{\Lambda_k^{-1}}$.

At the end of the $k$-th episode. The central controller receives all the Gram matrices $\{\epiGram_i^k\}_{i\in \cN}$ and the weighted feature sums $\{\epiFeatSum_i^k\}_{i\in \cN}$ for the $k$-th episode from the agents. The controller first checks if $\epiGram_i^k$ is symmetric (if differential privacy is required, it further checks $\norm{\epiGram_i^k}_F,\norm{\epiFeatSum_i^k}_2\le \epiLen$). If they are clearly corrupted, set $\epiGram_i^k = 0I$ and $\epiFeatSum_i^k = 0$. 
Then the controller updates the existing Gram matrices $\widehat \Gram_i^{k}$ and feature sums $\widehat \featSum_i^k$. 
The key to achieve Byzantine-robustness is the robust aggregation oracle (\Agg) which computes $\Lambda_k$ and $b_k$ from the sets $\{\widehat \Gram_i^{k} +\lambda_k I\}_{i\in \cN}$ and $\{\widehat \featSum_i^k\}_{i\in\cN}$. 

Then $\Lambda_{k+1}$ and the latest estimation $\theta_{k+1}$ are broadcast to all agents at the beginning of the $(k+1)$-th~episode.

\begin{algorithm}[t]
    \caption{The Byzantine-Robust Federated Linear UCB Framework}
    \label{alg:byzantine-ucb}
    \begin{algorithmic}[1]
    \Require $T$ the number of total time steps; $K$ the number of communication rounds;
    \Statex $\Agg$ the aggregation algorithm; $\{\lambda_k\}_{k=1}^K$ the regularization parameter;
    \Statex $\{\beta_k\}_{k=1}^K$ the confidence level; $\mu,\nu$ the privacy parameter;
    \State The central controller initializes $\Gram_i^1 = \widehat{\Gram}_i^1 = 0I$ and $\featSum_i^1 = \widehat{\featSum}_i^1 = 0$ for each $i\in\cN$
    \For{$k = 1, \ldots, \numEpi $}
    \State The controller computes $\Lambda_k = \Agg_{i\in \cN}(\widehat{\Gram}_i^k)+\lambda_k I$, $b_k = \Agg_{i\in\cN}(\widehat{\featSum}_i^k)$, and $\theta_k = \Lambda_k^{-1} b_k$
    \State The controller broadcasts $\theta_k$ and $\Lambda_k$ to all agents
    \For{each agent $i\in\cN$}
    \For{$\text{step } t\in \cT_k $}
    \State Receive a decision set $\cD_i^t$ from environment
    \State Select $x_i^t = \argmax_{x\in\cD_i^t} \langle x, \theta_k\rangle + \beta_k \|x\|_{\Lambda_k^{-1}}$
    \State Obtain $r_i^t$ from environment
    \EndFor
    \State Compute $\epiGram_i^k = \sum_{t\in \cT_k} x_i^t (x_i^t)^\top$ and $\epiFeatSum_i^k = \sum_{t\in \cT_k} x_i^t r_i^t$
    \State Send $\epiGram_i^k, \epiFeatSum_i^k$ to the controller
    \EndFor
    \For{each message $(\epiGram_i^k,\epiFeatSum_i^k)$ received from agent $i\in \cN$}
    \If{Differential privacy is required}
    \If{$\norm{\epiGram_i^k}_F > L $ \textbf{or} $ \epiGram_i^k$ is not symmetric \textbf{or} $\norm{\epiFeatSum_i^k}_2 > L$}
    \State The controller set $\epiGram_i^k = 0I, \epiFeatSum_i^k= 0$
    \EndIf
    \State The controller updates $\Gram_i^{k+1} = \Gram_i^k + \epiGram_i^k$ and $\featSum_i^{k+1} = \featSum_i^k + \epiFeatSum_i^k$
    \State The controller privatize $(\widehat{\Gram}_i^{k+1}, \widehat \featSum_i^{k+1})$ = \Pri($\Gram_i^{k+1}, \featSum_i^{k+1};\mu,\nu$) 
    \Else
    \State The controller set $\epiGram_i^k = 0I, \epiFeatSum_i^k= 0$ if $ \epiGram_i^k$ is not symmetric
    \State The controller updates $ \widehat{\Gram}_i^{k+1} = \Gram_i^{k+1} = \Gram_i^k + \epiGram_i^k$ and $\widehat \featSum_i^{k+1} =\featSum_i^{k+1} = \featSum_i^k + \epiFeatSum_i^k$
    \EndIf
    \EndFor
    
    \EndFor
    \end{algorithmic}
\end{algorithm}

Note that in Algorithm~\ref{alg:byzantine-ucb}, the function \Pri is executed by the controller other than by each agent locally as in \citep{dubey2020differentially}. This is because when an agent is Byzantine-attacked, even after privatizing the data it intends to send to the controller, the attacker can still deprivatize it or even send other private information to the controller. Since in our algorithm, the controller has access to the original data of agents without privatizing, we assume the controller is trustable for all agents. 

Furthermore, to tightly characterize how much the decisions taken $x_i^t$ vary from its expectation, we make the following assumption.
\begin{assumption}
    \label{assump:variance}
    Assume for every $t\in[T]$ and $i\in\cN_0^t$, we have with probability $1$,
    \begin{align*}
         \|x_i^t (x_i^t)^\top - \mathbb{E}[x_i^t (x_i^t)^\top ]\|_F^2 \le \sigmax^2,    
    \end{align*}
    where the randomness of $x_i^t = \argmax_{x\in \cD_i^t} ( \langle x, \theta^\ast \rangle + \beta_k \| x \|_{\Lambda_k^{-1}} )$ comes from the randomness in~$\cD_i^t$. 
\end{assumption}

All theorems henceforth holds under Assumption \ref{assump:variance}. Note that $\norm{x_i^t}_2\le 1$ directly implies $ \|x_i^t (x_i^t)^\top - \mathbb{E}[x_i^t (x_i^t)^\top ]\|_F^2 \le 4$. Thus in the worst case, $\sigma \leq 2$. On the other hand, if for every fixed $t\in[T]$, the decision sets for different agents are the same, we have $\sigma=0$.

\paragraph{The \byucb algorithm.}
\byucb instantiates the algorithmic framework (Algorithm \ref{alg:byzantine-ucb}) with $\Agg$ chosen to be an oracle that computes the exact geometric median of the input set and without the requirement of differential privacy.

Now we are ready to present our main theorem on the regret bound. For ease of exposition, we do not consider differential privacy and assume the geometric median can be exactly computed in this theorem. We will consider these two issues in Theorem~\ref{thm:regret_dp} in Section~\ref{sec:dp}.

\begin{theorem}[Robust regret bound of \byucb]
    \label{thm:regret}
    Let $C_\alpha = \frac{2-2\alpha}{1-2\alpha}$.
    For any given $\delta \in (0, 1)$, let $\iota = \log\left(\frac{128NT}{\delta}\right)$.
    Choose $\lambda_k =  \max\{\lambda_0,\lambda_1\sqrt{k}\}$ where $\lambda_0 = \epiLen$ and $\lambda_1=8\sqrt{\epiLen\iota} C_\alpha\sigmax$. 
    Choose
    \begin{align*}
        \beta_k 
        = 3\sqrt{\lambda_k d} 
        + \frac{4\sqrt{(k-1)\epiLen d\iota} C_\alpha (\sigmax+R) 
        }{\sqrt{\lambda_k}} + 2R \sqrt{\frac{d\iota}{N}}.   
    \end{align*}
    Then with probability at least $1-\delta$, the regret of \byucb is bounded by
    \begin{align*}
        R_T=& \cO\left(Rd\iota\sqrt{N T}+Nd\sqrt{T \iota}
        \left(\sqrt{\epiLen +C_\alpha \sigmax\sqrt{T\iota}}
        +\tfrac{\sqrt{T\iota}C_\alpha (\sigmax+R) }{\sqrt{\epiLen +C_\alpha \sigmax\sqrt{T\iota}}}\right)\right).
    \end{align*}
    In particular, if choosing $L = C_\alpha (\sigma + R) \sqrt{T \iota}$, then we have 
    \begin{align*}
        R_T  = \tilde \cO(dN T^{3/4}).
    \end{align*}
\end{theorem}

The above theorem demonstrates the relationship between the number of communication rounds $K$, the size of the corruption $\alpha$, the variance of the decision set $\sigma^2$ and the upper bound of the regret achieved by Algorithm \ref{alg:byzantine-ucb}. First, note that although the regret bound depends on $\alpha$, the algorithm is completely agnostic to it. Next, it is easy to see that the larger the size of the corruption or the variance term is, the larger the regret bound is. However, the dependence of regret upper bound on the number of communication rounds $K$ is more intricate. As long as $K=\Omega(\sqrt{T})$, the regret bound stays at $\tilde{\cO}(T^{3/4})$. However, if we further reduce the $K$, the regret bound will increase. Therefore, the best number of communication rounds without affect the convergence rate is $\Theta(\sqrt{T})$.

\haochuan{added comparison below}
\paragraph{Comparison to previous works.}
When there are no corruptions, several existing works (e.g. \citep{dubey2020differentially}) achieve an $\tilde{\cO}(\sqrt{T})$ which is nearly optimal. However, it is not clear whether the $\tilde{\cO}(T^{3/4})$ regret bound under Byzantine attacks is optimal and we leave answering the question as future work. Note that our regret bound does not reduce to the $\tilde{\cO}(\sqrt{T})$ when the level of corruption $\alpha$ goes to $0$ because the algorithm is agnostic to $\alpha$. We will provide a corruption level aware algorithm in Section~\ref{sec:mom} which improves the regret when an upper bound of $\alpha$ is known. The improved rate does reduce to $\tilde{\cO}(\sqrt{T})$ when $\alpha\to 0$.

The number of communications rounds in Theorem~\ref{thm:regret} is $\cO(\sqrt{T})$, whereas that in \citep{dubey2020differentially} is $\cO(N\log T)$. Our communication complexity does not depend on $N$ but has a worse dependence on $T$ compared to theirs. They are able to obtain a $\log T$ communication because in their algorithm, the agent adaptively decides when to communicate with the controller. However, when the agents can be arbitrarily attacked, we can not really allow the agent to decide the communication rounds. Otherwise, an attacked agent may choose to communicate every round which is highly communication-inefficient. 
Therefore, reducing the communication complexity becomes more challenging in presence of corruptions. We will see in Section~\ref{sec:dp} that it becomes even more challenging with privacy constraints.

%% file: 4_analysis.tex
\section{Proof sketch of the regret bound}
\label{sec:analysi}

In this section, we provide an overview of our analyses for Theorem~\ref{thm:regret}. For convenience, we introduce
\begin{align*}
	W_k = \sum\nolimits_{t=1}^{(k-1)\epiLen} \sum\nolimits_{i\in\cN_0} x_i^t (x_i^t)^\top, \quad 
	s_k = \sum\nolimits_{t=1}^{(k-1)\epiLen} \sum\nolimits_{i\in\cN_0} x_i^t r_i^t,
\end{align*} 
where the summations are over consistently noncorrupted agents only. Then the least square estimate of $\theta^\ast$ can be written as $\theta_{k}^\lse=W_k^{-1}s_k$ which is widely used in previous noncorrupted linear bandit algorithms~\citep{abbasi2011improved}. However, since we do not know $\cN_0$, the least square estimate is not computable. Instead, we use another estimate $\theta_k=\Lambda_k^{-1}b_k$ where $\Lambda_k$ and $b_k$ can be written as
\begin{align*}
    \Lambda_k = \lambda_k I + \GM_{i\in \cN}\left(\Gram_i^k\right)=\lambda_k I + \frac{W_k}{N_0} + E_k, \quad b_k = \GM_{i\in \cN}\left(\featSum_i^k\right)=\frac{s_k}{N_0} + e_k,
\end{align*}

where $\lambda_k>0$ is a time-varying regularization parameter to ensure the positive definiteness and also control the regret. $E_k$ and $e_k$ are the error terms of using geometric median instead of arithmetic mean:
\begin{align*}
	E_k \triangleq \GM_{i\in \cN}\left(\Gram_i^k\right)-\frac{1}{N_0}\sum_{t=1}^{(k-1)\epiLen} \sum_{i\in\cN_0} x_i^t (x_i^t)^\top, ~~~e_k \triangleq \GM_{i\in \cN}\left(\featSum_i^k\right)-\frac{1}{N_0}\sum_{t=1}^{(k-1)\epiLen} \sum_{i\in\cN_0} x_i^t r_i^t. 
\end{align*}
We will bound these two error terms in Lemma~\ref{lem:Eses}. Then we can bound the difference between $\theta_k$ and $\theta_k^\lse$ and thus bound the difference between $\theta_k$ and $\theta^\ast$.
\begin{lemma} \label{lem:Eses}
    Using the same parameter choices as in Theorem \ref{thm:regret}, with probability at least $1-\delta/2$, for all $k\in [K]$,
    \begin{align*}
            \|E_k\|_2\le 
            4 C_\alpha \sigmax\sqrt{(k-1)\epiLen \iota }, \quad\|e_k\|_2\le 
            4 C_\alpha (\sigmax+R)\sqrt{(k-1)\epiLen d\iota }.
    \end{align*}
\end{lemma}

With the bound on the divergence between the geometric median and mean in Lemma \ref{lem:Eses}, the difference between $\theta_k$ and $\theta^\ast$ can thus be bounded in the following lemma.

\begin{lemma}[Approximation error]  
\label{lem:ae}
    Using the same parameter choices as in Theorem \ref{thm:regret}, with probability at least $1-3\delta/4$, for all $x\in\R^d$ and $k\in [K]$,$
        \left|x^\top (\theta_k-\theta^\ast) \right| \le \beta_k \|x\|_{\Lambda_k^{-1}}$.
\end{lemma}

In previous papers like~\citep{abbasi2011improved}, constant $\beta$ is used to obtain a regret of $\tilde{\cO}(\sqrt{T})$. Here, Lemma~\ref{lem:ae} shows that $\beta_k = \tilde{\cO}(T^{1/4})$. Therefore, we can obtain an $\tilde{\cO}(T^{3/4})$ regret bound at best.

\subsection{Bounding the regret}
With some analyses standard for the \linucb algorithm~\citep{abbasi2011improved}, we can show 
\begin{align*}
    	R_T\le 2\beta_{\max} \sum\nolimits_{t=1}^{T}\sum\nolimits_{i\in\cN_0^t} \| x_i^t\|_{\Lambda_k^{-1}}
    	\le 2\beta_{\max} \sqrt{N T } \sqrt{\sum\nolimits_{t=1}^{T} \sum\nolimits_{i\in\cN_0^t}(x_i^t)^\top \Lambda_k^{-1} x_i^t},
    \end{align*}
where $\beta_{\max}=\max_{k\in[K]}\beta_k$. To get an $\tilde{\cO}(T^{3/4})$ regret bound, we also need to show 
\begin{align*}
    \sum\nolimits_{t=1}^{T} \sum\nolimits_{i\in\cN_0^t}(x_i^t)^\top \Lambda_k^{-1} x_i^t = \tilde{\cO}(1).
    \end{align*}
Previously works like \citep{abbasi2011improved} also bounded some quantity like this. But we have some additional issues to deal with. First, $\Lambda_k$ contains corrupted data. We use Lemma~\ref{lem:ae} to deal with this issue. In addition, the summation is taken over $\cN_0^t$ which is time-varying. We use a concentration argument to bound the difference between the summation over $\cN_0$ and that over $\cN_o^t$.
Finally, communication does not happen every step, which results in an additional error term. To deal with this issue, we need more careful analyses and to choose $\lambda_0=\epiLen$.

%% file: 5_dp.tex
\section{Differential privacy guarantees}
\label{sec:dp}
In this section, we first formally define differential privacy in our federated linear bandit problem with corruptions, and then discuss how to make our algorithm differentially private with the tree-based~mechanism. To guarantee differential privacy, we need to further make the following standard bounded reward assumption.

\begin{assumption}
\label{assump:bounded-reward}
$|r_i^t|\le1$ for every $t\in[T]$ and $i\in\cN_0^t$. 
\end{assumption}

 We assume that each user only trusts the controller and the agent that it is interacting with. Therefore, the decision sets $\{\cD_i^t\}_{t\in [T]}$ and rewards $\{r_i^t\}_{t\in[T]}$ of each agent $i$ must be made private to all other agents. Moreover, when an agent is attacked, it may also send private information to the controller, which also needs to be made private. No matter if agent $i$ is attacked or not, it suffices to make each update it sends to the server private to other agents. 

Therefore, instead of defining the dataset we want to make private as the collection of all $\{\cD_i^t\}_{t\in [T]}$ and $\{r_i^t\}_{t\in[T]}$ as in \citep{dubey2020differentially}, we define it as the collection of communication messages sent to the controller. For any $i\in\cN$, we view all messages sent by other agents (e.g.  $\{(U_j^k,u_j^k)\}_{k\in[K],j\in\cN, j\neq i}$ in Algorithm~\ref{alg:byzantine-ucb}) as a dataset for the algorithm $\cA_i$ on agent $i$. For any two datasets $\mathbf{S}_i$ and $\mathbf{S}'_i$, we say they are neighboring if they differ at only a single element. Formally, we define federated differential privacy w.r.t. communication as follows.

\begin{definition}[Federated differential privacy w.r.t. communication]
    For federated linear bandits with $N$ agents and $\dpepsilon,\dpdelta>0$, we say a randomized algorithm $\cA=\{\cA_i\}_{i\in\cN}$ is $(\dpepsilon,\dpdelta,N)$-federated differential private w.r.t. communication under continual multi-agent observations if for any $i\in\cN$ and datasets $\mathbf{S}_i$, $\mathbf{S}'_i$ of $\cA_i$ that are neighboring, it holds that for any subset of actions $S_i\subset \cD_i^1\times\cdots\times\cD_i^T$:
    \begin{align*}
        \Pr\left(\cA_i(\mathbf{S}_i)\in S_i\right)\le e^{\dpepsilon}\Pr\left(\cA_i(\mathbf{S}'_i)\in S_i\right)+\dpdelta.
    \end{align*}
\end{definition}

Note that for each pair $i,j\in\cN$ and $i\neq j$, algorithm $\cA_j$ accesses the data $\{(U_j^k,u_j^k)\}_{k\in[K]}$ only through the sequence $\{(\Gram_j^k,\featSum_j^k)\}_{k\in[K]}$. Therefore, it suffices to make $\{(\Gram_j^k,\featSum_j^k)\}_{k\in[K]}$ differentially private with respect to $\{(U_j^k,u_j^k)\}_{k\in[K]}$. Formally, we have the following lemma.
\begin{lemma}
If the sequence $\{(\Gram_j^k,\featSum_j^k)\}_{k\in[K]}$ is $(\dpepsilon,\dpdelta)$-differentially private with respect to $\{(U_j^k,u_j^k)\}_{k\in[K]}$ for every $j\in\cN$, then all agents are $(\dpepsilon,\dpdelta,N)$-federated differentially private w.r.t. communication.
\label{lem:dp_1}
\end{lemma}

Note that Lemma~\ref{lem:dp_1} is very similar to \citep[Proposition~3]{dubey2020differentially} except that we use a different dataset $\{(U_j^k,u_j^k)\}_{k\in[K]}$ than their $\{(x_j^t,r_j^t)\}_{t\in[T]}$. Nontheless, since $\{(\Gram_j^k,\featSum_j^k)\}_{k\in[K]}$ are partial sums of $\{(U_j^k,u_j^k)\}_{k\in[K]}$, we can still use the algorithm in \citep{dubey2020differentially} to privatize $\{\Gram_j^k,\featSum_j^k)\}_{k\in[K]}$ following the tree-based mechanism \citep{dwork2010differential}. However, since $\{(\Gram_j^k,\featSum_j^k)\}_{k\in[K]}$ are more sensitive to our dataset, we need to add noise with a variance $\cO(L^2)$ times as large as that of \cite{dubey2020differentially} in the tree-based mechanism. Effectively, the function \Pri in our algorithm returns 
\begin{align}
    \widehat{\Gram}_i^k = \Gram_i^k + \mnoise_i^k,\quad \widehat{\featSum}_i^k = \featSum_i^k + \vnoise_i^k, \label{alg:pri} 
\end{align}
where $\mnoise_i^k$ and $\vnoise_i^k$ are some Gaussian noise according to the tree-based mechanism.
To make the algorithm $(\dpepsilon,\dpdelta,N)$-federated differentially private, we have with probability $1-\delta/4$,
\begin{align*}
    \norm{\mnoise_i^k}_2, \norm{\vnoise_i^k}_2\le \mbd L\,,
\end{align*}
where $\mbd\triangleq48\iota \log(4/\dpdelta) \left(\sqrt{d}+\iota\right)/\dpepsilon$, for every $i\in\cN$ and $k\in [K]$. Note that the above bound is linear in $L$, the number of rounds between two communications. Therefore, the less frequently the communication happens, the harder to guarantee privacy, which makes it more challenging to reduce the communication complexity.

\paragraph{The \byucbdp algorithm.}
\byucbdp instantiates the algorithmic framework (Algorithm \ref{alg:byzantine-ucb}) by choosing $\Agg$ to be an oracle that computes the $\epsilon$-approximation of the  geometric median of the input set and \Pri to be the tree-based privatizing function~\eqref{alg:pri}.

Then \byucbdp has the following regret bound with differential privacy guarantees.

\begin{theorem}[Robust regret bound of \byucbdp]  \label{thm:regret_dp}
    Let $C_\alpha = \frac{2-2\alpha}{1-2\alpha}$.
    For any given $\delta \in (0, 1)$, let $\iota = \log\left(\frac{128NT}{\delta}\right)$. Given $\dpepsilon,\dpdelta>0$, let $        \mbd=48\iota \log(4/\dpdelta) \left(\sqrt{d}+\iota\right)/\dpepsilon$. Choose $\lambda_k = 2C_\alpha( \mbd L\sqrt{d}+\epsilonGM)+ \max\{\lambda_0,\lambda_1\sqrt{k}\}$ where $\lambda_0 = \epiLen$ and $\lambda_1=8\sqrt{\epiLen\iota} C_\alpha\sigmax$. 
    Choose
    \begin{align*}
        \beta_k 
        = 3\sqrt{\lambda_k d} 
        + \frac{4\sqrt{(k-1)\epiLen d\iota} C_\alpha (\sigmax+R) +C_\alpha( \mbd L+\epsilonGM)
        }{\sqrt{\lambda_k}} + 2R \sqrt{\frac{d\iota}{N}}.   
    \end{align*}
    Then with probability at least $1-\delta$, \byucbdp is $(\dpepsilon,\dpdelta,N)$-federated differentially private w.r.t. communication and its regret is bounded by

    \begin{align*}
        R_T = \cO\biggl(
        Rd\sqrt{N T} \iota  + \sqrt{C_{\alpha}}d N\sqrt{T \iota}
        \biggl(\sqrt{\mbd \epiLen \sqrt{d}+\epsilonGM + \sigmax\sqrt{T\iota}} +\tfrac{(\sigmax+R)\sqrt{T\iota} }{\sqrt{\mbd \epiLen \sqrt{d}+\epsilonGM + \sigmax\sqrt{T\iota}}}\biggr)\biggr).
    \end{align*}
In particular, if we choose $\epsilon$ small enough and $L = C_\alpha (\sigma + R) \sqrt{T \iota}$, then we have
\begin{align*}
    R_T = \tilde \cO\left( \left( \sqrt{\log (2/\nu) / \mu}  \right)   d^{3/2}NT^{3/4}  \right)\,.
\end{align*}
\end{theorem}

Compared to Theorem~\ref{thm:regret}, the dependence on $N$ and $T$ are the same. However, the dependence on $d$ increases from $\cO(d)$ to $\cO(d^{3/2})$ due to the noise $\mnoise_i^k$ and $\vnoise_i^k$ added to make the algorithm differentially private. It is worth noting that if we use a matrix geometric median oracle with respect to the spectral norm, the dependence can be improved to $\cO(d^{5/4})$ (discussed in Appendix~\ref{sec:app_proofs}), though this oracle can be harder to compute in practice.

%% file: 6_mom.tex
\section{Corruption level aware algorithm and its improved regret}
\label{sec:mom}

So far we use the geometric median oracle for robust aggregation (\Agg). One advantage of this choice is that it is agnostic to the corruption level $\alpha$. However, it may also turn into a disadvantage if we know $\alpha$ is small and we have a good estimate of its bound. This prior knowledge, if utilized effectively, can improve the performance of an algorithm. In this section, we introduce the geometric median of mean oracle, which can be viewed as an interpolation between geometric median and arithmetic mean. We then show that the resulting \byucbdpm algorithm, based on an approximate geometric median of mean oracle, has a better robust regret bound.

Suppose $\alpha\le 1/4$ is known. Right after the algorithm starts, the oracle randomly splits the set of all agents $\cN$ into $\groupnumber=3N_1$ groups, $\group=\{\cG_i\}_{i\in[\groupnumber]}$, as equally as possible. Therefore we have $\floor{\frac{1}{3\alpha}}\le \abs{\cG_i}\le \ceil{\frac{1}{3\alpha}}$. At each time $t\in[T]$, we say a group is noncorrupted if all its agents are noncorrupted; say it is corrupted otherwise. Then at the group level, we can also define the set of consistently noncorrupted groups $\group_0$ with $\groupnumber_0=\abs{\group_0}$ and its complement $\group_1$ with $\groupnumber_1=\abs{\group_1}$. Let $\gamma=\groupnumber_1/\groupnumber$ be the fraction of possibly corrupted groups. We know $\gamma\le 1/3$ and thus $C_{\gamma}\le 4$.

At each synchronization step, the oracle first computes arithmetic means within each group and then uses the approximate geometric median oracle to aggregate these means. Formally, given a set of vectors or matrices $\{z_j\}_{j\in\cN}$, the oracle returns
\begin{align}  \label{eqn:gmm}
    \GM^{\epsilonGM}_{1\le i\le P}\left(\frac{1}{\abs{\cG_i}}\sum_{j\in\cG_i}z_j\right).
\end{align}

\paragraph{The \byucbdpm algorithm.}
Aware of $\alpha$, \byucbdpm instantiates the algorithmic framework (Algorithm \ref{alg:byzantine-ucb}) by choosing $\Agg$ to be an $\epsilon$-approximate geometric median of mean oracle~\eqref{eqn:gmm} and \Pri to be the tree-based privatizing function~\eqref{alg:pri}.

Since geometric median of mean is an interpolation between geometric median and arithmetic mean, the error between the geometric median of mean and the mean is smaller. Therefore, we obtain the following tighter regret bound when $\alpha$ is small.

\begin{theorem}[Robust Regret Bound of \byucbdpm]  \label{thm:regret_mom}
    Suppose $\alpha\le 1/4$. For any given $\delta \in (0, 1)$, let $\iota = \log\left(\frac{128NT}{\delta}\right)$. Given $\dpepsilon,\dpdelta>0$, let $        \mbd=48\iota \log(4/\dpdelta) \left(\sqrt{d}+\iota\right)/\dpepsilon$. Choose $\lambda_k=8\left(\mbd \epiLen \sqrt{d}+\epsilonGM\right)+\max\left\{\lambda_0,\lambda_1\sqrt{k}\right\}$ where $\lambda_0 = \epiLen$ and $\lambda_1=128 \sigmax\sqrt{\alpha\epiLen\iota}$. 
    Choose
    \begin{align*}
        \beta_k=3\sqrt{\lambda_k d}+\frac{64 (\sigmax+R)\sqrt{\alpha(k-1)\epiLen d \iota } +4 \left(\mbd \epiLen+\epsilonGM \right)}{\sqrt{\lambda_k}}+2R\sqrt{\frac{d\iota}{N}}.   
    \end{align*}
    Then with probability at least $1-\delta$, \byucbdpm is $(\dpepsilon,\dpdelta,N)$-federated differentially private w.r.t. communication and its regret is bounded by
    \begin{align*}
        R_T = \cO\biggl(
        Rd\iota\sqrt{N T}+Nd\sqrt{T \iota}
        \biggl(\sqrt{\mbd \epiLen \sqrt{d}+\epsilonGM +\sigmax\sqrt{\alpha T\iota}}  +\tfrac{\sqrt{\alpha T\iota} (\sigmax+R) }{\sqrt{\mbd \epiLen \sqrt{d}+\epsilonGM + \sigmax\sqrt{\alpha T\iota}}}\biggr)\biggr).
    \end{align*}
In particular, if we choose $\epsilon$ small enough and $L = \max\{ (\sigma + R) \sqrt{\alpha T \iota},1 \}$, then we have
\begin{align*}
    R_T = \tilde \cO\left( \biggl(\sqrt{\log (2/\nu) / \mu} \biggr) d^{3/2}N\biggl(\alpha^{1/4}T^{3/4}  +\sqrt{T}\biggr)\right).
\end{align*}
\end{theorem}

This bound is much tighter than that in Theorem~\ref{thm:regret_dp} when $\alpha$ is very small. For example, when $\alpha=0$, the oracle reduces to the arithmetic mean oracle and the regret is bounded by $\tilde{\cO}(\sqrt{T})$, matching the lower bound in $T$ dependence.

%% file: 7_conclusion.tex
\section{Conclusion}
\label{sec:conclusion}

In this paper, we proposed a novel setup for linear bandit algorithms for use in modern recommendation systems with the following properties 
1) A federated learning architecture: the algorithm allows the data to be stored in local devices; 2) Robustness to Byzantine attacks: the algorithm has sublinear robust regret even if some agents send arbitrarily corrupted messages to the controller; and 3) Differential privacy: the agents do not want to reveal their identities to other agents. We present the \byucb algorithm to achieve the first two objectives and the \byucbdp algorithm to achieve all three objectives simultaneously with a slightly worse robust regret, both of which are agnostic to the proportion of corrupted agents. To complement the two algorithms, we propose a third algorithm \byucbdpm that meets all three objectives and takes advantage of a known small proportion of corrupted agents to obtain an improved robust regret. 


Our most general algorithm \byucbdp achieves $\tilde{\cO}(d^{3/2} N T^{3/4})$ regret. On the other hand, the information-theoretic lower bound on regret for single-agent linear bandits is known to be $\Omega(d\sqrt{T})$~\citep{lattimore2020bandit}, which immediately implies an $\Omega(d\sqrt{NT})$ lower bound on robust regret under the differential privacy constraint.
As we see, a gap exists for the dependence on all three parameters $d$, $N$ and $T$. Determining their optimal dependencies remains an open question. 

%% file: app_lemmas.tex

\section{Auxiliary lemmas}
\label{sec:useful_lemmas}

\begin{lemma}\label{lem:psd} Let $A,B\in\R^{d\times d}$ be two positive definite matrices. If $A\ge B$, we have $A^{-1}\le B^{-1}$.
\end{lemma}
\begin{proof}[Proof of Lemma~\ref{lem:psd}]
Since $A\ge B$, i.e., $A-B$ is positive semi-definite, we know $A-B$ can be decomposed as $A-B=M^\top M$ for some matrix $M\in\R^{d\times d}$. Then by Woodbury matrix identity, we know that\begin{align*}
    B^{-1}-A^{-1} &= B^{-1}M^\top (I+MB^{-1}M^\top)^{-1}MB^{-1}.
\end{align*}
Since $B$ is positive definite, we know $B^{-1}$ is positive definite. Thus $I+MB^{-1}M^\top$ is positive definite. Therefore $(I+MB^{-1}M^\top)^{-1}$ is positive definite. Then we have $B^{-1}-A^{-1}$ is positive semi-definite, i.e., $A^{-1}\le B^{-1}$.
\end{proof}

In general, the geometric median of symmetric matrices might not be symmetric. However, if we are given a geometric median of symmetric matrices, then we can always construct a symmetric one without any cost, as show in the following lemma.

\begin{lemma}
\label{lem:gm_sym}
Let $\{A_i\}_{i=1}^n$ be a set of symmetric matrices in $\R^{d\times d}$ and $\norm{\cdot}$ be a matrix norm on it satisfying $\norm{A}=\norm{A^\top}$ for all $A\in\R^{d\times d}$. Then if $\hat A_0$ is an $\epsilonGM$-approximate geometric median of $\{A_i\}_{i=1}^n$ with respect to $\norm{\cdot}$, then the symmetric matrix $\hat A = \frac{\hat A_0 + \hat A_0^\top}{2}$ is also an $\epsilonGM$-approximate geometric median of $\{A_i\}_{i=1}^n$ with respect to $\norm{\cdot}$.
\end{lemma}
\begin{proof}
Let $\hat{A}_0$ be an $\epsilonGM$-approximate geometric median of $\{A_i\}_{i=1}^n$ with respect to $\norm{\cdot}$. By definition of approximate geometric median, we have
\begin{align*}
    \frac{1}{n}\sum_{i=1}^n\norm{\hat{A}_0-A_i}\le \frac{1}{n}\min_{A\in\R^{d\times d}}\sum_{i=1}^n\norm{A-A_i}+\epsilonGM.
\end{align*}
By the property of $\norm{\cdot}$, we have
\begin{align*}
    \frac{1}{n}\sum_{i=1}^n\norm{\hat{A}_0^\top-A_i}=\frac{1}{n}\sum_{i=1}^n\norm{\hat{A}_0-A_i}.
\end{align*}
Let $\hat{A}=(\hat{A}_0+\hat{A}_0^\top)/2$. Since $\norm{\cdot}$ is convex, by Jensen's inequality, we have
\begin{align*}
    \frac{1}{n}\sum_{i=1}^n\norm{\hat{A}-A_i}\le \frac{1}{2}\left(\frac{1}{n}\sum_{i=1}^n\norm{\hat{A}_0^\top-A_i}+\frac{1}{n}\sum_{i=1}^n\norm{\hat{A}_0-A_i}\right)\le \frac{1}{n}\min_{A\in\R^{d\times d}}\sum_{i=1}^n\norm{A-A_i}+\epsilonGM.
\end{align*}
Therefore the symmetric matrix $\hat{A}$ is an $\epsilonGM$-approximate geometric median of $\{A_i\}_{i=1}^n$ with respect to $\norm{\cdot}$ and we complete the proof.
\end{proof}

\begin{remark}
The Frobenius norm and spectral norm satisfy the requirement in Lemma~\ref{lem:gm_sym}. Thus for these two norms, we can symmetrize the outcome for any oracle that computes an $\epsilonGM$-approximate geometric median of symmetric matrices.
\end{remark}



The following lemma modifies \cite[Lemma~2]{wu2020federated} and can be used to bound the difference between geometric median and arithmetic mean if choosing $z_0$ as the arithmetic mean of noncorrupted vectors.

\begin{lemma}
\label{lem:concentration_gm}
Let $\{z_i\}_{i\in\cN}$ be a set of vectors or matrices in an Euclidean space with norm $\norm{\cdot}$. Let $\hat{z}$ be their $\epsilonGM$-approximate geometric median. Let $\cN_1\subseteq \cN$ with $\alpha=\abs{\cN_1}/\abs{\cN}<1/2$. For any fixed $z_0$, we have
\begin{align*}
    \norm{\hat{z}-z_0}\le C_{\alpha}\left( \frac{\sum_{i\not\in\cN_1}\norm{z_i-z_0}}{\abs{\cN}-\abs{\cN_1}}+\epsilonGM\right),
\end{align*}
where $C_\alpha=\frac{2-2\alpha}{1-2\alpha}$.
\end{lemma}
\begin{proof}[Proof of Lemma~\ref{lem:concentration_gm}]
    Note that by reverse triangle inequality, for every $i\in\cN_1$, $\norm{\hat{z}-z_i}\ge \norm{z_i-z_0}-\norm{\hat{z}-z_0}$; and that for every $i\not\in\cN_1$, $\norm{\hat{z}-z_i}\ge \norm{\hat{z}-z_0}-\norm{z_i-z_0}$. Summing up $\norm{\hat{z}-z_i}$ over all $i\in\cN$ gives
    \begin{align*}
        \sum_{i\in\cN}\norm{\hat{z}-z_i}\ge \sum_{i\in\cN}\norm{z_i-z_0}-2\sum_{i\not\in\cN_1}\norm{z_i-z_0}+\left(\abs{\cN}-2\abs{\cN_1}\right)\norm{\hat{z}-z_0}.
    \end{align*}
    By the definition of approximate geometric median, we have
    \begin{align*}
    \sum_{i\in\cN}\norm{\hat{z}-z_i}\le \sum_{i\in\cN}\norm{z_0-z_i}+\epsilonGM \abs{\cN}.
    \end{align*}
    Combining these two inequalities, we have
    \begin{align*}
        \norm{\hat{z}-z_0}\le \frac{2\sum_{i\not\in\cN_1}\norm{z_i-z_0}+\epsilonGM \abs{\cN}}{\abs{\cN}-2\abs{\cN_1}}\le C_{\alpha}\left( \frac{\sum_{i\not\in\cN_1}\norm{z_i-z_0}}{\abs{\cN}-\abs{\cN_1}}+\epsilonGM\right).
    \end{align*}
\end{proof}

%% file: app_proofs.tex
\section{Proofs of the regret bounds}
\label{sec:app_proofs}
\subsection{Proof of Proposition~\ref{prop:lower_bound}}
\label{app:pf-prop}

\begin{proof}[Proof of Proposition~\ref{prop:lower_bound}]
    Consider the setting that $\cD_i^t=\{-1,1\}$ for every $t\in[T]$ and $i\in\cN_0^t$. Let $\abs{\theta^\ast}=1$. For any $i\in\cN_1$, at each time $t$, agent $i$ is attacked with probability $1/2$ independently with other time steps or other agents. When it is attacked, it receives a reward according to a fake parameter $\theta^{\text{fake}}=-\theta^\ast$. Then any algorithm on agent $i$ which does not communicate with other agents or the controller cannot distinguish between the two models $\theta^\ast=1$ or $\theta^\ast=-1$. Then we can always choose one of them to make the algorithm perform no better than a random guess. Therefore the regret on agent $i$ is at least $T$. Noting that there can be $\alpha N$ such agents, we complete the proof.
\end{proof}

\subsection{Proofs of Theorem~\ref{thm:regret} and Theorem~\ref{thm:regret_dp}}
In this subsection, we prove Theorem~\ref{thm:regret_dp} which reduces to Theorem~\ref{thm:regret} when $B=\epsilonGM=0$.

We first extend Lemma~\ref{lem:Eses} to the following lemma which further considers approximate geometric median and differential privacy. Note that Lemma~\ref{lem:Eses_ex} reduces to Lemma~\ref{lem:Eses} when $\mbd =\epsilonGM= 0$ and thus it suffices to prove Lemma~\ref{lem:Eses_ex}.
\begin{lemma} \label{lem:Eses_ex}
Using the same parameter choices as in Theorem \ref{thm:regret_dp}, with probability at least $1-\delta/2$, for all $k\in [K]$,
\begin{align*}
        \|E_k\|_2&\le 
        4 C_\alpha \sigmax\sqrt{(k-1)\epiLen \iota } +C_\alpha \left(\mbd \epiLen\sqrt{d} +\epsilonGM \right),
        \\
        \|e_k\|_2&\le 
        4 C_\alpha (\sigmax+R)\sqrt{(k-1)\epiLen d\iota } +C_\alpha \left(\mbd \epiLen +\epsilonGM \right).
\end{align*}
\end{lemma}
\begin{proof}[Proof of Lemma~\ref{lem:Eses_ex}]
    We first bound $\|E_k\|_2$. We have with probability at least $1-\delta/4$,
    \begin{align*}
     \norm{E_k}_F =\,    &\left\|{\GM^{\epsilonGM}_{i\in \cN}\left(V_i^k+\mnoise_i^t\right)-\frac{1}{N_0}\sum_{i\in\cN_0} V_i^k }\right\|_F\\
             \stackrel{(i)}{\le} \, & C_\alpha \left( \frac{1}{N_0}\sum_{i\in\cN_0}\left\| V_i^k+\mnoise_i^t-\frac{1}{N_0}\sum_{i\in\cN_0} V_i^k
        \right\|_F +\epsilonGM \right)\\
        \stackrel{(ii)}{\le}\, & C_\alpha \left( \frac{1}{N_0}\sum_{i\in\cN_0}\left( \norm{V_i^k-\E V_i^k}_F+\norm{\mnoise_i^t}_F
        \right) + \left\|\frac{1}{N_0}\sum_{i\in\cN_0} V_i^k -\E V_i^k\right\|_F+\epsilonGM \right)\\
        \stackrel{(iii)}{\le}\, & C_\alpha \left( \frac{2}{N_0}\sum_{i\in\cN_0} \norm{V_i^k-\E V_i^k}_F +\mbd \epiLen\sqrt{d} +\epsilonGM \right)
        \\
        =\, & \frac{2C_\alpha}{N_0}\sum_{i\in\cN_0} \left\| \sum_{t=1}^{(k-1)\epiLen} \left(x_i^t(x_i^t)^\top -\E[x_i^t(x_i^t)^\top]\right)\right\|_F+C_\alpha \left(\mbd \epiLen\sqrt{d} +\epsilonGM \right),
    \end{align*}
    where $(i)$ is due to Lemma~\ref{lem:concentration_gm}, $(ii)$ is by triangle inequality, and to obtain $(iii)$, we apply Jensen inequality on the convex operator $\norm{\cdot}_F$ and use the fact that $\norm{\mnoise_i^t}_F\le \sqrt{d}\norm{\mnoise_i^t}_2\le \mbd \epiLen\sqrt{d} $ with probability at least $1-\delta/4$.
    
    Then by \citep[Theorem~1.8]{hayes2005large}, with probability at least $1-\delta_{i,k}^{(1)}$ for fixed $i\in\cN_0$ and $k\in [K]$,
    \begin{align*}
        \left\| \sum_{t=1}^{(k-1)\epiLen} \left(x_i^t(x_i^t)^\top -\E[x_i^t(x_i^t)^\top]\right)\right\|_F^2 \le 2(k-1)\epiLen\sigmax^2\log\left(\frac{2e^2}{\delta_{i,k}^{(1)} }\right).
    \end{align*}
    Choosing $\delta_{i,k}^{(1)}=\frac{\delta}{8N_0\numEpi }$ and by union bound, we know with probability at least $1-\delta/8$, the above inequality holds for every $i\in\cN_0$ and $k\in [\numEpi ]$. Therefore we have
    \begin{align*}
        \norm{E_k}_2\le \|E_k\|_F \le 2 C_\alpha \sigmax\sqrt{2(k-1)\epiLen \log \left(\frac{16e^2N_0\numEpi }{\delta}\right) } +C_\alpha \left(\mbd \epiLen\sqrt{d} +\epsilonGM \right).
    \end{align*}
    Now let us bound $\|e_k\|_2$. We can similarly obtain that 
    \begin{align*}
        \norm{e_k}_2=\left\|\GM^\epsilonGM_{i\in \cN}(v_i^k+\vnoise_i^t)-\frac{1}{N_0}\sum_{i\in\cN_0}v_i^k\right\|_2 \le \frac{2C_\alpha}{N_0}\sum_{i\in\cN_0} \left\| \sum_{t=1}^{(k-1)\epiLen} \left(x_i^t r_i^t -\E[x_i^t r_i^t]\right)\right\|_2+C_\alpha \left(\mbd \epiLen +\epsilonGM \right),
    \end{align*}
    where we can further bound
    \begin{align*}
        \left\| \sum_{t=1}^{(k-1)\epiLen} \left(x_i^t r_i^t -\E[x_i^t r_i^t]\right)\right\|_2 &= \left\| \sum_{t=1}^{(k-1)\epiLen} \left(x_i^t (x_i^t)^\top \theta^\ast + x_i^t\eta_i^t-\E[x_i^t (x_i^t)^\top] \theta^\ast\right)\right\|_2\\
        &\le \sqrt{d}\left\| \sum_{t=1}^{(k-1)\epiLen} \left(x_i^t(x_i^t)^\top -\E[x_i^t(x_i^t)^\top]\right)\right\|_F+ \left\| \sum_{t=1}^{(k-1)\epiLen} x_i^t\eta_i^t\right\|_2.
    \end{align*}
    We have already bounded the first term. For the second term, according to \citep[Theorem~1]{abbasi2011improved}, we have with probabiliy $1-\delta_{i,k}^{(2)}$ for fixed $i\in\cN_0$ and $k\in [K]$,
    \begin{align*}
        \left\| \sum_{t=1}^{(k-1)\epiLen}  x_i^t\eta_i^t\right\|_2^2 \le 2(k-1)\epiLen \left\|  \sum_{t=1}^{(k-1)\epiLen}  x_i^t\eta_i^t\right\|_{\left((k-1)\epiLen+\sum_{t=1}^{(k-1)\epiLen}  x_i^t(x_i^t)^\top \right)^{-1}}^2 \le 4(k-1)\epiLen R^2 d\log\left(\frac{2}{\delta_{i,k}^{(2)}}\right).
    \end{align*}
    Choosing $\delta_{i,k}^{(2)}=\frac{\delta}{8N_0\numEpi }$, we have
    \begin{align*}
        \|e_k\|_2&\le 4 C_\alpha (\sigmax+R)\sqrt{(k-1)\epiLen d\iota } +C_\alpha \left(\mbd \epiLen +\epsilonGM \right).
    \end{align*}
    Also note that by union bound, the total probability of all failures we consider in this lemma is less than $\delta/2$. We complete the proof.
\end{proof}

We also extend Lemma~\ref{lem:ae} to the following lemma which also considers approximate geometric median and differential privacy, i.e., it uses the parameter choices as in Theorem \ref{thm:regret_dp}.
\begin{lemma}[Approximation error]  
\label{lem:ae_dp}
    Using the same parameter choices as in Theorem \ref{thm:regret_dp}, with probability at least $1-3\delta/4$, for all $x\in\R^d$ and $k\in [K]$, we have $
        \left|x^\top (\theta_k-\theta^\ast) \right| \le \beta_k \|x\|_{\Lambda_k^{-1}}$.
\end{lemma}
\begin{proof}[Proof of Lemma~\ref{lem:ae_dp}]
    For every $x\in\R^d$, we have
    \begin{align*}
    	x^\top (\theta_k-\theta^\ast) &= x^\top\Lambda_k^{-1}b_k-x^\top\Lambda_k^{-1}\Lambda_k\theta^\ast\\
    	&=x^\top\Lambda_k^{-1}\frac{s_k}{N_0}+x^\top\Lambda_k^{-1}e_k-x^\top\Lambda_k^{-1}\left(\lambda_k I+E_k+\frac{W_k}{N_0}\right)\theta^\ast\\
    	&=-x^\top\Lambda_k^{-1}\left(\lambda_k I+E_k\right)\theta^\ast+x^\top\Lambda_k^{-1}e_k+\frac{x^\top\Lambda_k^{-1}}{N_0}\left(s_k-W_k\theta^\ast\right)\\
    	&\triangleq R_1+R_2+R_3.
    \end{align*}
    Then we bound the three terms separately. First, we choose
    $$\lambda_k = 2C_\alpha \left(\mbd \epiLen\sqrt{d} +\epsilonGM \right) + \max\{\lambda_0,\lambda_1\sqrt{k}\},$$ 
    where $\lambda_0 = \epiLen$ and $\lambda_1=8\sqrt{\epiLen\iota} C_\alpha\sigmax$. Then we can guarantee $\lambda_k\ge 2\|E_k\|_2$ according to Lemma~\ref{lem:Eses_ex}. Also we can assume $\Lambda_k$ is symmetric (otherwise symmetrize it as in Lemma~\ref{lem:gm_sym}). Furthermore, it is obviously positive definite since $\lambda_k\ge 2\|E_k\|_2$. Therefore $\norm{\cdot}_{\Lambda_k^{-1}}$ is well-defined. Then we have
    \begin{align*}
    	|R_1|=& \left|x^\top\Lambda_k^{-1}\left(\lambda_k I+E_k\right)\theta^\ast \right|\\
    	\le & \sqrt{\left| (\theta^\ast)^\top \left(\lambda_k I+E_k\right)^{\top} \Lambda_k^{-1} \left(\lambda_k I+E_k\right)\theta^\ast\right|} \norm{x}_{\Lambda_k^{-1}}   \\
    	\le & 3\sqrt{\lambda_kd}\|x\|_{\Lambda_k^{-1}}.
    \end{align*}
    Similarly, we can bound the second term
    \begin{align*}
    	|R_2|\le \frac{4 C_\alpha (\sigmax+R)\sqrt{(k-1)\epiLen d\iota } +C_\alpha \left(\mbd \epiLen +\epsilonGM \right) }{\sqrt{\lambda_k}} \|x\|_{\Lambda_k^{-1}}.
    \end{align*}
    To bound $R_3$, first note that
    \begin{align*}
        s_k-{W}_k\theta^\ast = \sum_{t=1}^{(k-1)\epiLen} \sum_{i\in\cN_0} x_i^t\left( r_i^t-(x_i^t)^\top\theta^\ast\right)
        = \sum_{t=1}^{(k-1)\epiLen} \sum_{i\in\cN_0} x_i^t \eta_i^t.
    \end{align*}
    By \citep[Theorem~1]{abbasi2011improved}, we know that with probability at least $1-\delta/4$, for every $k$,
    \begin{align*}
    	\left\|s_k-{W}_k\theta^\ast\right\|^2_{({W}_k+\lambda_0N_0I/2)^{-1}}\le 2R^2\left(d\log\left(1+2T/\lambda_0\right)-\log(4/\delta)\right)\le 2R^2d\iota.
    \end{align*}
    Note that since $\lambda_k\ge2\norm{E_k}_2$, we have
    \begin{align*}
        N_0 \Lambda_k - {W}_k 
        = N_0 (\lambda_k I+E_k)
        \ge \frac{N_0}{2} \lambda_k I
        \ge \frac{N_0}{2} \lambda_0 I.
    \end{align*}
    By Lemma~\ref{lem:psd}, we have
    \begin{align*}
        \left({W}_k+\frac{N_0}{2} \lambda_0 I\right)^{-1}\ge \frac{1}{N_0} \Lambda_k^{-1}. 
    \end{align*}
    Then we have
    \begin{align*}
    	|R_3|\le \frac{1}{N_0}\|x\|_{\Lambda_k^{-1}}\left\|s_k-W_k\theta^\ast\right\|_{\Lambda_k^{-1}}
    	\le \frac{1}{\sqrt{N_0}}\|x\|_{\Lambda_k^{-1}}	\left\|s_k-{W}_k\theta^\ast\right\|_{({W}_k+\lambda_0N_0I/2)^{-1}}
    	\le 2R\sqrt{\frac{d\iota}{N}} \|x\|_{\Lambda_k^{-1}},
    \end{align*}
    
    Thus we can show that for every $x\in\R^d$,
    \begin{align*}
        \left|x^\top (\theta_k-\theta^\ast) \right|\le |R_1|+|R_2|+|R_3| \le \beta_k \|x\|_{\Lambda_k^{-1}},
    \end{align*}
    where $\beta_k = 3\sqrt{\lambda_k d}+\frac{4 C_\alpha (\sigmax+R)\sqrt{(k-1)\epiLen d\iota } +C_\alpha \left(\mbd \epiLen +\epsilonGM \right) }{\sqrt{\lambda_k}}+2R\sqrt{\frac{d\iota}{N}}$.
\end{proof}

Now we are ready to prove Theorem~\ref{thm:regret_dp} which reduces to Theorem~\ref{thm:regret} when $B=\epsilonGM=0$.
\begin{proof}[Proof of Theorem~\ref{thm:regret_dp}]
    Let $x_{i,t}^\ast\triangleq\argmax_{x\in\cD_i^t} \langle x, \theta^\ast\rangle$ be the optimal action. According to the algorithm, 
    $x_i^t = \argmax_{x\in\cD_i^t} \langle x, \theta_k\rangle + \beta_k \|x\|_{\Lambda_k^{-1}}$. Therefore,
    \begin{align*}
    	\langle x_i^t,\theta_k\rangle +\beta_k \|x_i^t\|_{\Lambda_k^{-1}} \ge \langle x_{i,t}^\ast,\theta_k\rangle+\beta_k \| x_{i,t}^\ast\|_{\Lambda_k^{-1}}.
    \end{align*}
    By Lemma~\ref{lem:ae},
    \begin{align*}
    	\left|\langle x_i^t,\theta_k\rangle-\langle x_i^t,\theta^\ast\rangle\right|&\le \beta_k \|x_i^t\|_{\Lambda_k^{-1}},\\
    	\left|\langle x_{i,t}^\ast,\theta_k\rangle-\langle x_{i,t}^\ast,\theta^\ast\rangle\right|&\le \beta_k \|x_{i,t}^\ast\|_{\Lambda_k^{-1}}.
    \end{align*}
    Combining the above three inequalities, we obtain
    \begin{align*}
    	\langle x_{i,t}^\ast,\theta^\ast\rangle-\langle x_i^t,\theta^\ast\rangle \le 2\beta_k \| x_i^t\|_{\Lambda_k^{-1}}\le 2 \beta_{\max}  \| x_i^t\|_{\Lambda_k^{-1}},
    \end{align*}
    where we define $\beta_{\max}\triangleq \max_{k\in[K]}\beta_k$. Therefore the regret is bounded by
    \begin{align*}
    	R_T\le 2 \beta_{\max} \sum_{k=1}^{\numEpi} \sum_{t\in\cT_k}\sum_{i\in\cN_0^t} \| x_i^t\|_{\Lambda_k^{-1}}
    	\le 2 \beta_{\max} \sqrt{N T } \sqrt{\sum_{k=1}^{\numEpi} \sum_{t\in\cT_k} \sum_{i\in\cN_0^t}(x_i^t)^\top \Lambda_k^{-1} x_i^t}.
    \end{align*}
    Note that
    \begin{align*}
        &\quad\sum_{k=1}^{\numEpi} \sum_{t\in\cT_k}\sum_{i\in\cN_0^t} (x_i^t)^\top \Lambda_k^{-1} x_i^t \\
        &=  \sum_{k=1}^{\numEpi} \sum_{t\in\cT_k}\left[\frac{N_0^t}{N_0}\sum_{i\in\cN_0} (x_i^t)^\top \Lambda_k^{-1} x_i^t+N_0^t\epiLen  \left(\frac{1}{N_0^t\epiLen }\sum_{i\in\cN_0^t} (x_i^t)^\top \Lambda_k^{-1} x_i^t-\frac{1}{N_0\epiLen }\sum_{i\in\cN_0} (x_i^t)^\top \Lambda_k^{-1} x_i^t\right)\right]\\
        &\le 2 \sum_{k=1}^{\numEpi} \sum_{t\in\cT_k}\sum_{i\in\cN_0} (x_i^t)^\top \Lambda_k^{-1} x_i^t+ \sum_{k=1}^{\numEpi} \sum_{t\in\cT_k} \xi_t,
    \end{align*}
    where
    \begin{align*}
        \xi_t\triangleq& N_0^t  \left(\frac{1}{N_0^t }\sum_{i\in\cN_0^t} (x_i^t)^\top \Lambda_k^{-1} x_i^t-\frac{1}{N_0 }\sum_{i\in\cN_0} (x_i^t)^\top \Lambda_k^{-1} x_i^t\right)\\
        =&N_0^t  \left(\frac{1}{N_0^t }\sum_{i\in\cN_0^t} \tr\left(\Lambda_k^{-1}\left[x_i^t(x_i^t)^\top -\E[x_i^t(x_i^t)^\top] \right]\right)-\frac{1}{N_0 }\sum_{i\in\cN_0} \tr\left(\Lambda_k^{-1}\left[x_i^t(x_i^t)^\top -\E[x_i^t(x_i^t)^\top] \right]\right)\right)
    \end{align*}
    Note that by Assumption~\ref{assump:variance}, we have
    \begin{align*}
        \tr\left(\Lambda_k^{-1}\left[x_i^t(x_i^t)^\top -\E[x_i^t(x_i^t)^\top] \right]\right)\le \sigmax d \norm{\Lambda_k^{-1}}_2. 
    \end{align*}
    By our i.i.d. assumption on decision sets, it is straight forward to verify that $\xi_t$ is a $\left(3\sigmax^2 d^2 N\norm{\Lambda_k^{-1}}_2^2\right)$-subGaussian random variable. Then by Azuma’s inequality for martingales with subGaussian tails~\citep[Theorem~2]{Shamir2011AVO}
    , with probability at least $1-\delta/4$,
    \begin{align*}
       \abs{ \sum_{k=1}^{\numEpi} \sum_{t\in\cT_k} \xi_t}\le& \sigmax d\sqrt{6NL\sum_{k=1}^K \norm{\Lambda_k^{-1} }_2^2\log\frac{8}{\delta}}\le 6\sigmax d\sqrt{NL\iota}\sqrt{\sum_{k=1}^K \lambda_k^{-2}}\\\le& 6\sigmax d\sqrt{NL\iota}\cdot \frac{\sqrt{\iota}}{\lambda_1}\le d\sqrt{N\iota}\le Nd\sqrt{\iota}.
    \end{align*}
    Then we can obtain that
    \begin{align*}
    	R_T\le& 4 \beta_{\max} \sqrt{NT}\sqrt{\sum_{k=1}^{\numEpi} \sum_{t\in\cT_k}\sum_{i\in\cN_0} (x_i^t)^\top \Lambda_k^{-1} x_i^t +Nd\sqrt{\iota} }\\
    	=&4 \beta_{\max} \sqrt{NT}\sqrt{ \sum_{k=1}^{\numEpi} \sum_{t\in\cT_k}\sum_{i\in\cN_0} (x_i^t)^\top \Lambda_{k+1}^{-1} x_i^t +\sum_{k=1}^{\numEpi} \sum_{t\in\cT_k}\sum_{i\in\cN_0} (x_i^t)^\top (\Lambda_k^{-1}-\Lambda_{k+1}^{-1}) x_i^t +Nd\sqrt{\iota}}.
    \end{align*}
    
    Note that 
    \begin{align*}
    	\sum_{k=1}^{\numEpi} \sum_{t\in\cT_k}\sum_{i\in\cN_0} (x_i^t)^\top (\Lambda_k^{-1}-\Lambda_{k+1}^{-1}) x_i^t =& \sum_{k=1}^{\numEpi} \sum_{t\in\cT_k}\sum_{i\in\cN_0} \|\Lambda_k^{-1}-\Lambda_{k+1}^{-1}\|_2\\
    	\le &\sum_{k=1}^{\numEpi} \sum_{t\in\cT_k}\sum_{i\in\cN_0} \tr\left(\Lambda_k^{-1}-\Lambda_{k+1}^{-1}\right)\\
    	=&\epiLen N_0 \tr\left(\Lambda_1^{-1}-\Lambda_{\numEpi }^{-1}\right)
    	\le \frac{\epiLen N_0 d}{\lambda_0}.
    \end{align*}
    Also, by \citep[Lemma~11]{abbasi2011improved}, we can bound
    \begin{align*}
    	\sum_{k=1}^{\numEpi} \sum_{t\in\cT_k}\sum_{i\in\cN_0} (x_i^t)^\top \Lambda_{k+1}^{-1} x_i^t & \le N_0 \sum_{k=1}^{\numEpi} \sum_{t\in\cT_k}\sum_{i\in\cN_0} (x_i^t)^\top ({W}_{k+1}+\lambda_0N_0/2)^{-1} x_i^t\\
    	&\le 2N_0d\log\left(1+2T/\lambda_0\right)
    	\le 2N_0d\iota.
    \end{align*}
    Choosing $\lambda_0 = \epiLen$ and combining all the inequalities above, we get the regret bound with probability at least $1-\delta$, 
    \begin{align*}
        R_T\le& 32Rd\iota\sqrt{N T}+200Nd\sqrt{T \iota}
        \left(\sqrt{2C_\alpha(\mbd \epiLen \sqrt{d}+\epsilonGM)+\max\left\{\epiLen ,8C_\alpha \sigmax\sqrt{T\iota}\right\}}\right.\\
        &\left.+\frac{\sqrt{T\iota}C_\alpha (\sigmax+R) }{\sqrt{2C_\alpha(\mbd \epiLen \sqrt{d}+\epsilonGM)+\max\left\{\epiLen ,8C_\alpha \sigmax\sqrt{T\iota}\right\}}}\right).
    \end{align*}
\end{proof}

%% file: app_proof_MoM.tex
\subsection{Proof of Theorem~\ref{thm:regret_mom}}
\label{sec:app_mom}
If we use the geometric median of mean oracle, the estimate is $\theta_k=\Lambda_k^{-1}b_k$ where $\lambda_k$ and $b_k$ can be written as 
\begin{align*}
\Lambda_k =& \lambda_k I + \GM^{\epsilonGM}_{1\le i\le P}\left(\frac{1}{\abs{\cG_i}}\sum_{j\in\cG_i}(\Gram_j^k+\mnoise_j^k)\right)=\lambda_k I + \frac{W_k}{N_0} + E_k^{\alpha},\\ \quad b_k =&\GM^{\epsilonGM}_{1\le i\le P}\left(\frac{1}{\abs{\cG_i}}\sum_{j\in\cG_i}(\featSum_j^k+\vnoise_j^k)\right)=\frac{s_k}{N_0} + e_k^{\alpha},    
\end{align*}
where $E_k^\alpha$ and $e_k^\alpha$ are the error terms of using geometric median of mean instead of arithmetic mean:
\begin{align*}
    E_k^\alpha\triangleq& \GM^{\epsilonGM}_{1\le i\le P}\left(\frac{1}{\abs{\cG_i}}\sum_{j\in\cG_i}(\Gram_j^k+\mnoise_j^k)\right)-\frac{1}{N_0}\sum_{i\in\cN_0}\Gram_i^k,\\
    e_k^\alpha \triangleq& \GM^{\epsilonGM}_{1\le i\le P}\left(\frac{1}{\abs{\cG_i}}\sum_{j\in\cG_i}(\featSum_j^k+\vnoise_j^k)\right)-\frac{1}{N_0}\sum_{i\in\cN_0}\featSum_i^k.
\end{align*}
Similar to Lemma~\ref{lem:Eses_ex}, we can bound these two error terms in the following lemma:

\begin{lemma} \label{lem:Eses_ex_mom}
Using the same parameter choices as in Theorem \ref{thm:regret_mom}, with probability at least $1-\delta/2$, for all $k\in [K]$,
\begin{align*}
        \|E_k^\alpha\|_2&\le 
        64 \sigmax\sqrt{\alpha(k-1)\epiLen \iota } +4 \left(\mbd \epiLen\sqrt{d} +\epsilonGM \right),
        \\
        \|e_k^\alpha\|_2&\le 
        64 (\sigmax+R)\sqrt{\alpha(k-1)\epiLen d \iota } +4 \left(\mbd \epiLen+\epsilonGM \right).
\end{align*}
\end{lemma}
\begin{proof}[Proof of Lemma~\ref{lem:Eses_ex_mom}]
    We first bound $\|E_k^\alpha\|_2$. We have with probability at least $1-\delta/4$,
    \begin{align*}
     \norm{E_k^\alpha}_F =\,    &\left\|{\GM^{\epsilonGM}_{1\le i\le P}\left(\frac{1}{\abs{\cG_i}}\sum_{j\in\cG_i}(\Gram_j^k+\mnoise_j^k)\right)-\frac{1}{N_0}\sum_{i\in\cN_0} V_i^k }\right\|_F\\
             \stackrel{(i)}{\le} \, & C_\gamma \left( \frac{1}{\groupnumber_0}\sum_{i\in\group_0}\left\| \frac{1}{\abs{\cG_i}}\sum_{j\in\cG_i}(\Gram_j^k+\mnoise_j^k)-\frac{1}{N_0}\sum_{i\in\cN_0} V_i^k
        \right\|_F +\epsilonGM \right)\\
        \stackrel{(ii)}{\le}\, & C_\gamma \left( \left( \frac{1}{\groupnumber_0}\sum_{i\in\group_0}\left\| \frac{1}{\abs{\cG_i}}\sum_{j\in\cG_i}\Gram_j^k-\E[V_j^k]
        \right\|_F 
        \right) + \left\|\frac{1}{N_0}\sum_{i\in\cN_0} V_i^k -\E[V_i^k]\right\|_F+BL\sqrt{d}+\epsilonGM \right)\\
        =\,&
          C_\gamma\left( \frac{1}{\groupnumber_0}\sum_{i\in\group_0}\left\| \frac{1}{\abs{\cG_i}}\sum_{j\in\cG_i} \sum_{t=1}^{(k-1)\epiLen} \left(x_j^t(x_j^t)^\top -\E[x_j^t(x_j^t)^\top]\right)
        \right\|_F 
        \right) \\&+ \left\|\frac{C_\gamma}{N_0}\sum_{i\in\cN_0} \sum_{t=1}^{(k-1)\epiLen} \left(x_i^t(x_i^t)^\top -\E[x_i^t(x_i^t)^\top]\right)\right\|_F+C_\gamma\left(BL\sqrt{d}+\epsilonGM \right)
    \end{align*}
    where $(i)$ is due to Lemma~\ref{lem:concentration_gm}, $(ii)$ is by triangle inequality and the fact that $\norm{\mnoise_j^t}_F\le \sqrt{d}\norm{\mnoise_j^t}_2\le \mbd \epiLen\sqrt{d} $ with probability at least $1-\delta/4$.
    
    Then by \citep[Theorem~1.8]{hayes2005large}, with probability at least $1-\delta_{i,k}^{(1)}$ for fixed $i\in\cP_0$ and $k\in [K]$,
    \begin{align*}
        \left\|\sum_{j\in\cG_i} \sum_{t=1}^{(k-1)\epiLen} \left(x_j^t(x_j^t)^\top -\E[x_j^t(x_j^t)^\top]\right)\right\|_F^2 \le 2(k-1)\epiLen\sigmax^2\abs{\cG_i}\log\left(\frac{2e^2}{\delta_{i,k}^{(1)} }\right).
    \end{align*}
    Similarly, with probability at least $1-\delta_{k}^{(1)}$, 
    \begin{align*}
        \left\|\sum_{i\in\cN_0} \sum_{t=1}^{(k-1)\epiLen} \left(x_i^t(x_i^t)^\top -\E[x_i^t(x_i^t)^\top]\right)\right\|_F^2\le 2(k-1)\epiLen\sigmax^2N_0\log\left(\frac{2e^2}{\delta_{k}^{(1)} }\right).
    \end{align*}
    Choosing $\delta_{i,k}^{(1)}=\delta_{k}^{(1)}=\frac{\delta}{8(\groupnumber_0+1)\numEpi }$ and by union bound, we know with probability at least $1-\delta/8$, the above two inequalities hold for every $i\in\cN_0$ and $k\in [\numEpi ]$. Therefore we have
    \begin{align*}
        \norm{E_k^\alpha}_2\le \|E_k^\alpha\|_F \le& \frac{2 C_\gamma}{\sqrt{\abs{\cG_i}}} \sigmax\sqrt{2(k-1)\epiLen \iota } +C_\gamma \left(\mbd \epiLen\sqrt{d} +\epsilonGM \right)\\
        \le&64\sigmax\sqrt{\alpha(k-1)\epiLen \iota }+4\left(\mbd \epiLen\sqrt{d} +\epsilonGM \right),
    \end{align*}
    where the last inequality is due to $C_\gamma\le 4$ and $$\frac{1}{\sqrt{\abs{\cG_i}}}\le \sqrt{\frac{1}{1/(3\alpha)-1}}\le 4\sqrt{\alpha}.$$
    Now let us bound $\|e_k^\alpha\|_2$. We can similarly obtain that 
    \begin{align}
    \label{eq:ek_bd}
        \norm{e_k^\alpha}_2=&\left\|\GM^{\epsilonGM}_{1\le i\le P}\left(\frac{1}{\abs{\cG_i}}\sum_{j\in\cG_i}(\featSum_j^k+\vnoise_j^k)\right)-\frac{1}{N_0}\sum_{i\in\cN_0}\featSum_i^k\right\|_2\nonumber\\
        \le& C_\gamma\left( \frac{1}{\groupnumber_0}\sum_{i\in\group_0}\left\| \frac{1}{\abs{\cG_i}}\sum_{j\in\cG_i} \sum_{t=1}^{(k-1)\epiLen} \left(x_j^t r_j^t -\E[x_j^t r_j^t]\right)
        \right\|_F 
        \right)\nonumber \\&+ \left\|\frac{C_\gamma}{N_0}\sum_{i\in\cN_0} \sum_{t=1}^{(k-1)\epiLen} \left(x_i^t r_j^t -\E[x_i^t r_j^t]\right)\right\|_F+C_\gamma\left(BL\sqrt{d}+\epsilonGM \right).
    \end{align}
    Similar to the proof of Lemma~\ref{lem:Eses_ex}, we can bound that
    \begin{align*}
        \left\| \sum_{j\in\cG_i}\sum_{t=1}^{(k-1)\epiLen} \left(x_j^t r_j^t -\E[x_j^t r_j^t]\right)\right\|_2 &= \left\| \sum_{j\in\cG_i}\sum_{t=1}^{(k-1)\epiLen} \left(x_j^t (x_j^t)^\top \theta^\ast + x_j^t\eta_j^t-\E[x_j^t (x_j^t)^\top] \theta^\ast\right)\right\|_2\\
        &\le \sqrt{d}\left\|\sum_{j\in\cG_i} \sum_{t=1}^{(k-1)\epiLen} \left(x_j^t(x_j^t)^\top -\E[x_j^t(x_j^t)^\top]\right)\right\|_F+ \left\|\sum_{j\in\cG_i} \sum_{t=1}^{(k-1)\epiLen} x_j^t\eta_j^t\right\|_2.
    \end{align*}
    We have already bounded the first term. For the second term, according to \citep[Theorem~1]{abbasi2011improved}, we have with probability $1-\frac{\delta}{8(\groupnumber_0+1)\numEpi }$ for fixed $i\in\cN_0$ and $k\in [K]$,
    \begin{align*}
        \left\|\sum_{j\in\cG_i} \sum_{t=1}^{(k-1)\epiLen} x_j^t\eta_j^t\right\|_2^2 \le& 2(k-1)\epiLen \abs{\cG_i} \left\|\sum_{j\in\cG_i} \sum_{t=1}^{(k-1)\epiLen} x_j^t\eta_j^t\right\|^2_{\left((k-1)\epiLen\abs{\cG_i}+\sum_{j\in\cG_i}\sum_{t=1}^{(k-1)\epiLen}  x_j^t(x_j^t)^\top \right)^{-1}}\\
        \le& 4(k-1)\abs{\cG_i}\epiLen R^2 d\log\left(\frac{2}{\delta_{i,k}^{(2)}}\right).
    \end{align*}
    Similarly, we can bound the second term of \eqref{eq:ek_bd} with probability $1-\frac{\delta}{8(\groupnumber_0+1)\numEpi }$. Combining these inequalities, we have
    \begin{align*}
        \|e_k^\alpha\|_2&\le 
        64 (\sigmax+R)\sqrt{\alpha(k-1)\epiLen d \iota } +4 \left(\mbd \epiLen+\epsilonGM \right).
    \end{align*}
    Also note that by union bound, the total probability of all failures we consider in this lemma is less than $\delta/2$. We complete the proof.
\end{proof}

With Lemma~\ref{lem:Eses_ex_mom}, we can bound the difference between $\theta_k$ and $\theta^\ast$ for \byucbdpm.
\begin{lemma}[Approximation error]  
\label{lem:ae_mom}
    Using the same parameter choices as in Theorem \ref{thm:regret_mom}, with probability at least $1-3\delta/4$, for all $x\in\R^d$ and $k\in [K]$, we have $
        \left|x^\top (\theta_k-\theta^\ast) \right| \le \beta_k \|x\|_{\Lambda_k^{-1}}$.
\end{lemma}
Lemma~\ref{lem:ae_mom} and its proof are the same as Lemma~\ref{lem:ae_dp} except that we choose different parameters $\lambda_k$ and $\beta_k$ here. Now we are ready to prove Theorem~\ref{thm:regret_mom}.
\begin{proof}[Proof of Theorem~\ref{thm:regret_mom}]
The proof is the same as that of Theorem~\ref{thm:regret_dp} except that we use different parameter choices and that we need to bound the following term more carefully:
\begin{align*}
    I_\alpha\triangleq \abs{\sum_{k=1}^K\sum_{t\in\cT_k}\xi_t}.
\end{align*}
First, when $\alpha=0$, by the definition of $\xi_t$, we know $\xi_t=0$ for every $t\in[T]$ and thus $I_0=0.$ When $\alpha>0$, by its definition, we must have $\alpha\ge 1/N$. Then following the proof of Theorem~\ref{thm:regret_dp}, we can bound that with probability at least $1-\delta/4$,
\begin{align*}
    I_\alpha\le 6\sigma d\sqrt{NL\iota}\cdot\frac{\sqrt{\iota}}{\lambda_1}\le d\sqrt{N\iota/\alpha}\le Nd\sqrt{\iota}.
\end{align*}
This bound is tight enough for the theorem. We can complete the proof following that of Theorem~\ref{thm:regret_dp}.
\end{proof}

